\documentclass{article}
\usepackage{iclr2027_conference,times}

\usepackage[utf8]{inputenc}
\usepackage[T1]{fontenc}
\usepackage{amsmath,amssymb,amsfonts,amsthm,mathtools}
\usepackage{algorithm}
\usepackage{algpseudocode}
\usepackage{booktabs}
\usepackage{graphicx}
\usepackage{microtype}
\usepackage{xcolor}
\usepackage{url}
\usepackage{hyperref}
\usepackage[nameinlink,noabbrev]{cleveref}

\newtheorem{definition}{Definition}
\newtheorem{theorem}{Theorem}
\newtheorem{proposition}{Proposition}
\newtheorem{corollary}{Corollary}

\newtheorem{remark}{Remark}

\newcommand{\cS}{\mathcal{S}}
\newcommand{\cA}{\mathcal{A}}
\newcommand{\cZ}{\mathcal{Z}}
\newcommand{\cF}{\mathcal{F}}
\newcommand{\cB}{\mathcal{B}}
\newcommand{\E}{\mathbb{E}}
\newcommand{\R}{\mathbb{R}}
\newcommand{\1}{\mathbf{1}}
\newcommand{\doop}{\operatorname{do}}
\newcommand{\rank}{\operatorname{rank}}
\newcommand{\argmax}{\operatorname*{arg\,max}}
\newcommand{\CQM}{\textsc{CQM}}

\title{Counterfactual Quotient Models:\\
Learning What Actions Change, Not What the World Does}

\author{%
  Junlin Chen \quad Ruijie Wang \quad Jianxin Li \\
  School of Computer Science and Engineering, Beihang University \\
  Beijing, China \\
  \texttt{\{23371194,ruijiew,lijx\}@buaa.edu.cn}
}

\iclrfinalcopy

\begin{document}

\maketitle
\lhead{Preprint}
\begin{abstract}
Reinforcement-learning models commonly predict complete future states,
observations, or feature occupancies, even though action selection depends only
on differences between the consequences of candidate actions.  As a result,
these models may devote substantial statistical and representational capacity
to high-dimensional phenomena that evolve independently of the agent's current
choice.  We introduce the \emph{Counterfactual Quotient Model} (\CQM), which
treats action-conditioned futures as equivalent when they differ only by a
component shared across actions.  Its canonical centered representation removes
this common component while preserving every pairwise action comparison
expressible by the modeled reward family.  The implemented model learns these
action-dependent effects directly from synchronized counterfactual rollouts,
so shared stochastic dynamics cancel before function approximation rather than
after complete futures have been predicted.  We establish the decision
sufficiency, identifiability, common-mode invariance, approximation behavior,
and regret properties of the resulting representation.  Controlled experiments
in physics-based environments provide initial evidence for these properties:
direct effect learning suppresses action-independent variation, supports
previously unseen reward queries, and improves action ranking relative to
models trained to predict absolute futures.
\end{abstract}

\section{Introduction}

The conventional modeling question in reinforcement learning is
\emph{what will happen next?}  World models answer it by predicting states,
observations, rewards, or latent trajectories and using those predictions for
planning or policy learning
\citep{sutton1991dyna,ha2018worldmodels,hafner2019planet,
hafner2020dream,hafner2025dreamer}.  Successor representations and successor
features summarize expected future occupancy
\citep{dayan1993successor,barreto2017successor,barreto2018transfer}.  These
approaches are powerful, but complete future prediction is generally a
stronger requirement than action selection.

An agent choosing among actions needs to know how their consequences differ.
Predictive content shared by every candidate action cannot change the current
ranking, yet an absolute prediction objective still spends samples, decoder
capacity, and optimization effort on that shared content.  The mismatch can
be severe in rich environments: lighting, weather, background motion, other
agents, and autonomous physical processes may dominate observation variance
while remaining unaffected by the current action.  Representation-learning
methods can suppress distractors or preserve reward-relevant structure
\citep{ferns2004metrics,gelada2019deepmdp,zhang2021invariant,
wang2022denoised,kemertas2021towards}, but their abstractions usually act on
states or are tied to specified rewards.

This paper instead quotients the action-indexed future itself.  Two predictive
models are treated as decision-equivalent whenever their difference is the
same for every action.  The resulting \emph{counterfactual quotient} retains
all pairwise action contrasts and discards only an action-independent
baseline.  Its canonical representative describes which future events become
more or less likely under each intervention.  Integrating a reward against
that representative recovers the corresponding centered action value.

Our contributions are:
\begin{itemize}
  \item We formulate action-conditioned prediction as a quotient of successor
  measures and identify a canonical signed representative that preserves
  action comparisons while removing shared future content.
  \item We establish measure-level effect--value duality and minimal decision
  sufficiency, pairwise identification, common-mode invariance, exact
  paired-noise cancellation, feature-level approximation and rank bounds, and
  decision-regret guarantees.
  \item We present a practical paired-rollout objective using common random
  numbers and a low-rank \CQM{} architecture whose outputs satisfy the
  quotient constraint by construction.
  \item We evaluate the construction in DeepMind Control Suite
  environments augmented with high-dimensional common dynamics, comparing
  \CQM{} with equally sized absolute predictors and reward-aware baselines
  under a common experimental protocol.
\end{itemize}

\section{Related Work}

\paragraph{World models and predictive control.}
Model-based RL traditionally learns transition or observation models and
uses them for simulation, planning, or policy optimization, from Dyna
\citep{sutton1991dyna} to learned latent simulators.  Recurrent World Models
\citep{ha2018worldmodels}, the Predictron \citep{silver2017predictron}, latent
planning methods \citep{hafner2019planet}, model-based policy optimization
\citep{janner2019trust}, MuZero \citep{schrittwieser2020muzero}, and
Dreamer-style agents \citep{hafner2020dream,hafner2025dreamer} demonstrate the
effectiveness of compact predictive states.  Their learned object remains an
absolute future, even when represented in a latent space.  \CQM{} instead
identifies and removes the entire subspace of action-independent future
predictions.  This change is orthogonal to the choice of deterministic,
probabilistic, recurrent, or latent architecture.

\paragraph{Successor representations and transfer.}
The successor representation encodes discounted future state occupancy
\citep{dayan1993successor}; successor features generalize this idea to
features and factor a linear reward from predictive dynamics
\citep{barreto2017successor,barreto2018transfer}.  For feature reward
$r_w=w^\top\phi$,
successor features provide $Q_w=\psi^\top w$ and support transfer across
reward weights.  At its general level, \CQM{} centers the complete successor
measure and is queryable by any bounded reward; successor features arise only
after projection through $\phi$.  The practical feature \CQM{} preserves the
usual reward-linear readout but learns only the action-contrastive part of
$\psi$.  Crucially, direct quotient learning centers the target before a
capacity constraint is imposed.

\paragraph{Decision-aware and value-equivalent models.}
The value-equivalence principle argues that models should preserve Bellman
updates for specified function and policy classes rather than reconstruct the
environment \citep{grimm2020value}.  Value-aware model learning and
task-conditioned critics, including universal value function approximators
\citep{schaul2015universal}, share the goal of retaining information useful
for control.  Their equivalence relation depends on selected rewards, values,
or policies.  The measure \CQM{} is sufficient for all bounded rewards on its
event space; its finite projection is sufficient for linear rewards over the
chosen features.  In the dual view developed below, a task-conditioned value
model learns queries $w^\top\kappa$, whereas feature \CQM{} learns the primal
effect vector $\kappa$ once.

\paragraph{State abstraction and nuisance removal.}
Bisimulation-based abstractions preserve reward and transition similarity
\citep{ferns2004metrics,ferns2011bisimulation}, and the broader state
abstraction literature classifies which distinctions may be removed while
preserving control \citep{li2006unified}.  DeepMDP
\citep{gelada2019deepmdp}, invariant representations
\citep{zhang2021invariant}, and robust bisimulation metric learning
\citep{kemertas2021towards} seek compact, control-relevant state
representations.  Denoised MDPs explicitly factor controllability and reward
relevance to remove distractors \citep{wang2022denoised}.  These methods
primarily quotient or compress the \emph{state}.  Our quotient acts on
\emph{action-indexed future predictions}: two large and observably different
futures are identified whenever their difference is shared by all candidate
actions.  Moreover, no reward is needed to define action influence.

\paragraph{Counterfactual credit assignment.}
Hindsight and counterfactual methods estimate how an earlier action
contributed to a later outcome.  We use intervention notation in the standard
causal sense \citep{pearl2009causality}, while making no claim that the
environmental causal graph is learned.  Counterfactual Credit Assignment
separates action effects from external factors using future-conditioned
critics \citep{mesnard2021counterfactual}.  COCOA measures contribution
coefficients toward reward outcomes to reduce policy-gradient variance
\citep{meulemans2023cocoa}.  These methods target reward-specific credit or
gradient estimation.  Measure \CQM{} instead defines the complete
reward-independent signed future effect.  Its density ratio
\cref{eq:hindsight-ratio} reveals a shared estimation structure, while the
modeled object and intended reuse differ.

\paragraph{Invariances and reward equivalence.}
Potential-based shaping identifies reward transformations that preserve
optimal policies \citep{ng1999policy}.  This is a quotient over reward
functions.  Our construction instead quotients action-indexed successor
measures by action-independent signed-measure additions.  Both viewpoints
separate decision content from a non-identifiable or behaviorally irrelevant
``gauge,'' but they act on different mathematical objects.

\section{Counterfactual Quotient Model}

\subsection{signed successor measures}

Consider a discounted controlled process with state space $\cS$, finite action
space $\cA$, transition kernel $P$, discount $\gamma\in[0,1)$, and a fixed
continuation policy $\pi$, using standard Markov decision process notation
\citep{puterman1994markov}.  Let
$Z_t=(S_t,A_t,S_{t+1})\in\cZ$ denote a transition event.  For a finite horizon
$H$, define the discounted action-interventional successor measure
\begin{equation}
  M^\pi_{H,s,a}(B)
  =
  \E\left[
    \sum_{k=0}^{H-1}\gamma^k
    \mathbf 1\{Z_{t+k}\in B\}
    \,\middle|\,
    s_t=s,\ \doop(a_t=a),\
    a_{t+k}\sim\pi\ (k\geq 1)
  \right].
  \label{eq:successor-measure}
\end{equation}
Its total mass is
$m_H=\sum_{k=0}^{H-1}\gamma^k$, independent of $s$ and $a$.  The
infinite-horizon measure has mass $(1-\gamma)^{-1}$ when $\gamma<1$.

The event space can be specialized without changing the construction.  For a
state-only reward, one may take $Z_t=S_{t+1}$; including
$(S_t,A_t,S_{t+1})$ also covers transition- and action-dependent rewards.
Only the first action is intervened on.  Every later action is sampled from
the same fixed $\pi$, so $M^\pi_{H,s,a}$ answers the controlled question
``what discounted future follows if the first action is $a$ and behavior
thereafter follows $\pi$?''  Thus the quotient is reward-independent but
policy-dependent.

Let $\rho(\cdot\mid s)$ be a full-support reference distribution over the
first action.  Define the reference successor measure and the
\emph{counterfactual effect measure}
\begin{equation}
  \overline M^\pi_{H,s}
  =
  \sum_b\rho(b\mid s)M^\pi_{H,s,b},
  \qquad
  K^\pi_{H,s,a}
  =
  M^\pi_{H,s,a}-\overline M^\pi_{H,s}.
  \label{eq:effect-measure}
\end{equation}
Because its positive and negative parts have equal total mass,
\begin{equation}
  K^\pi_{H,s,a}(\cZ)=0,
  \qquad
  \sum_a\rho(a\mid s)K^\pi_{H,s,a}=0.
  \label{eq:measure-zero}
\end{equation}
Thus $K$ is generally not a probability measure.  It is a finite signed
measure: positive mass marks future events made more likely by action $a$,
and negative mass marks events suppressed by that action.

Let $\mathfrak M$ be the vector space of finite signed measures on $\cZ$, and
let $\mathfrak F$ contain action-indexed maps
$\mathcal M:\cS\times\cA\to\mathfrak M$.  Define the action-independent
subspace
\[
  \mathfrak B
  =
  \{\mathcal C\in\mathfrak F:
  \mathcal C_{s,a}=C_s\ \text{for every }a\}.
\]

\begin{definition}[Measure-level counterfactual quotient]
Two action-indexed successor measures $\mathcal M_1,\mathcal M_2$ are
counterfactually equivalent if
$\mathcal M_1-\mathcal M_2\in\mathfrak B$.  The counterfactual quotient is
$\mathfrak F/\mathfrak B$, and \cref{eq:effect-measure} is its
$\rho$-centered representative.
\end{definition}

The use of the signed-measure envelope $\mathfrak M$ is essential.  Successor
measures themselves are positive, but subtracting two equal-mass positive
measures produces a signed measure.  Quotienting in the ambient vector space
makes action-independent additions and subtraction well defined.  Choosing
$\rho$ fixes a gauge: different full-support references produce different
centered representatives of the same pairwise action contrasts.

For any bounded measurable reward $r:\cZ\to\R$, define
\begin{equation}
  Q^\pi_{r,H}(s,a)
  =
  \int_{\cZ}r(z)\,\mathrm dM^\pi_{H,s,a}(z).
  \label{eq:measure-value}
\end{equation}
The measure-level quotient is therefore independent of a particular reward
while remaining queryable by arbitrary bounded rewards.

\subsection{finite feature projection}

Directly parameterizing a signed measure on a rich event space is generally
impractical.  Let $\phi:\cZ\to\R^d$ be an integrable outcome feature map and
project the successor and effect measures:
\begin{align}
  F^\pi_H(s,a)
  &=
  \int_{\cZ}\phi(z)\,\mathrm dM^\pi_{H,s,a}(z),\\
  \kappa^\pi_H(s,a)
  &=
  \int_{\cZ}\phi(z)\,\mathrm dK^\pi_{H,s,a}(z)\nonumber\\
  &=
  F^\pi_H(s,a)-\sum_b\rho(b\mid s)F^\pi_H(s,b).
  \label{eq:feature-projection}
\end{align}
For every reward weight $w\in\R^d$, the feature reward
$r_w(z)=w^\top\phi(z)$ satisfies
\begin{equation}
  Q^\pi_{w,H}(s,a)=w^\top F^\pi_H(s,a),
  \qquad
  A^\pi_{w,\rho,H}(s,a)=w^\top\kappa^\pi_H(s,a).
  \label{eq:q-readout}
\end{equation}
The implemented model is this finite-dimensional projection.  In the
experiments, $\rho$ and $\pi$ are uniform.

\begin{remark}[Measure and feature levels]
The two levels answer different reward classes.  The signed measure $K$ can
be integrated against any bounded measurable reward on $\cZ$.  The vector
$\kappa\in\R^d$ retains only the $d$ moments selected by $\phi$ and is exact
for rewards in
$\{z\mapsto w^\top\phi(z):w\in\R^d\}$.  Thus the practical model is not
claimed to preserve every nonlinear reward.  \Cref{prop:reward-error} later
quantifies the loss when a reward is only approximately linear in $\phi$.
\end{remark}

\subsection{Projected quotient and gauge}

Let $\cF$ be the vector space of functions
$F:\cS\times\cA\rightarrow\R^d$ and let
\begin{equation}
  \cB
  =
  \left\{
    B\in\cF:
    B(s,a)=b(s)\ \text{for some }b:\cS\rightarrow\R^d
  \right\}
  \label{eq:gauge-subspace}
\end{equation}
be the subspace of action-independent functions.

\begin{definition}[Projected counterfactual quotient]
Two future models $F,G\in\cF$ are counterfactually equivalent, written
$F\sim G$, if $F-G\in\cB$.  Their equivalence class $[F]$ belongs to the
quotient space $\cF/\cB$.  This is the image of the measure quotient under
the linear feature projection whenever $F$ and $G$ arise from successor
measures.
\end{definition}

Every member of $[F]$ has the same pairwise action differences.  We select a
canonical representative with the centering operator
\begin{equation}
  (\mathcal C_\rho F)(s,a)
  =
  F(s,a)-\sum_{b\in\cA}\rho(b\mid s)F(s,b).
  \label{eq:center}
\end{equation}
The finite-dimensional counterfactual effect is
\begin{equation}
  \kappa^\pi_H=\mathcal C_\rho F^\pi_H,
  \qquad
  \sum_a\rho(a\mid s)\kappa^\pi_H(s,a)=0.
  \label{eq:kappa}
\end{equation}
It retains every action contrast but no action-independent baseline.

\subsection{Model parameterization}

Let $f_\theta(s,a)\in\R^d$ be an unconstrained action-conditioned network.
\CQM{} applies centering as its final layer:
\begin{equation}
  \kappa_\theta(s,a)
  =
  f_\theta(s,a)
  -
  \sum_b\rho(b\mid s)f_\theta(s,b).
  \label{eq:cqm-param}
\end{equation}
The quotient constraint therefore holds exactly for every parameter value,
not merely at convergence.

Our implementation first encodes $s$ with a two-layer multilayer perceptron.
The state code is concatenated with an action embedding and mapped to a
rank-$m$ code $z_\theta(s,a)\in\R^m$.  A shared linear decoder
$D\in\R^{d\times m}$ produces $f_\theta(s,a)=Dz_\theta(s,a)+c$.  Centering
removes $c$ and any state-dependent component shared by the raw action heads.
The explicit rank bottleneck exposes the distinction between allocating
capacity to absolute futures and to action effects.

\subsection{Paired counterfactual supervision}

Suppose a simulator, learned generative model, or resettable environment can
produce synchronized branches.  For an initial state $s_i$, sample one
sequence of exogenous random variables and one continuation-action sequence
from $\pi$.  Reuse both sequences for every candidate first action.  Let
$\widehat F_i(a)$ denote the realized discounted feature sum.  The paired
target is
\begin{equation}
  \widehat\kappa_i(a)
  =
  \widehat F_i(a)
  -
  \sum_b\rho(b\mid s_i)\widehat F_i(b).
  \label{eq:paired-target}
\end{equation}
We optimize
\begin{equation}
  \mathcal L_{\mathrm{CQM}}(\theta)
  =
  \frac{1}{N|\cA|}
  \sum_{i=1}^{N}\sum_{a\in\cA}
  \left\|
    \kappa_\theta(s_i,a)-\widehat\kappa_i(a)
  \right\|_2^2.
  \label{eq:cqm-loss}
\end{equation}
For uniform $\rho$, minimizing this centered loss is equivalent, up to a
constant factor, to matching all pairwise differences
$f_\theta(s,a)-f_\theta(s,b)$ to
$\widehat F(s,a)-\widehat F(s,b)$.  The pairwise form is useful when only a
subset of actions can be branched; \cref{thm:pairwise} gives its population
identification property.
The equivalence follows from the finite-sample identity
\begin{equation}
  \sum_{a\in\cA}\|u_a-\bar u\|_2^2
  =
  \frac{1}{2|\cA|}
  \sum_{a,b\in\cA}\|u_a-u_b\|_2^2,
  \qquad
  \bar u=\frac{1}{|\cA|}\sum_a u_a.
  \label{eq:center-pairwise-identity}
\end{equation}
Applying \cref{eq:center-pairwise-identity} to
$u_a=f_\theta(s,a)-\widehat F(s,a)$ shows that centered regression and
complete pairwise-difference regression have exactly the same minimizers.
For a new reward weight $w$, no representation update is required:
\begin{equation}
  \widehat A_{w,\rho}(s,a)
  =
  w^\top\kappa_\theta(s,a),
  \qquad
  \widehat a_w(s)=\argmax_a\widehat A_{w,\rho}(s,a).
  \label{eq:decision}
\end{equation}

\begin{algorithm}[t]
  \caption{\CQM-B: branched feature-projection training}
  \label{alg:cqm}
  \begin{algorithmic}[1]
    \Require States $\{s_i\}_{i=1}^N$, continuation policy $\pi$,
    reference $\rho$, horizon $H$, model $f_\theta$
    \For{$i=1,\ldots,N$}
      \State Sample shared process noise $\xi_{i,0:H-1}$
      and shared continuation actions $a_{i,1:H-1}\sim\pi$
      \For{$a\in\cA$}
        \State Roll out from $s_i$ with first action $a$, shared
        $\xi_{i,0:H-1}$, and shared $a_{i,1:H-1}$
        \State Store discounted outcome $\widehat F_i(a)$
      \EndFor
      \State $\widehat\kappa_i(a)\gets
      \widehat F_i(a)-\sum_b\rho(b\mid s_i)\widehat F_i(b)$ for all $a$
    \EndFor
    \State Minimize \cref{eq:cqm-loss}, with model outputs centered as in
    \cref{eq:cqm-param}
    \State For any task $w$, select actions using \cref{eq:decision}
  \end{algorithmic}
\end{algorithm}

\subsection{Ordinary-trajectory density-ratio form}

The signed-measure definition also suggests an estimator when state cloning is
unavailable.  Assume
$M^\pi_{H,s,a}\ll\overline M^\pi_{H,s}$.  Its Radon--Nikodym derivative is
\begin{equation}
  h^\pi_{H,s,a}(z)
  =
  \frac{\mathrm dK^\pi_{H,s,a}}
       {\mathrm d\overline M^\pi_{H,s}}(z)
  =
  \frac{\mathrm dM^\pi_{H,s,a}}
       {\mathrm d\overline M^\pi_{H,s}}(z)-1.
  \label{eq:density-ratio}
\end{equation}
To estimate it, sample $A\sim\rho(\cdot\mid s)$ and then sample a discounted
future event $Z$ from the normalized successor measure associated with $A$.
Bayes' rule gives
\begin{equation}
  h^\pi_{H,s,a}(z)
  =
  \frac{p(A=a\mid s,Z=z)}{\rho(a\mid s)}-1.
  \label{eq:hindsight-ratio}
\end{equation}
Thus an outcome-conditioned action classifier provides an ordinary-trajectory
estimator, denoted \CQM-H.  Reward queries become
\[
  A^\pi_{r,\rho,H}(s,a)
  =
  \int r(z)h^\pi_{H,s,a}(z)
  \,\mathrm d\overline M^\pi_{H,s}(z).
\]
\CQM-H shares an estimation technique with hindsight credit assignment and
COCOA, but its target is the complete reward-independent signed future
measure.  The present experiments evaluate only the cleaner branched
feature-projection estimator \CQM-B; \cref{eq:hindsight-ratio} is included to
show that the quotient object is not conceptually restricted to resettable
simulators.

\section{Experiments}
\label{sec:physics-benchmark}

\subsection{Benchmark environments and observations}

We replace the constructed controllable dynamics with four state-based
environments from the DeepMind Control Suite \citep{tassa2018deepmind}, whose
physics is simulated by MuJoCo \citep{todorov2012mujoco}: Cartpole Swingup,
Reacher Easy, Cheetah Run, and Walker Walk.  Unlike the Distracting Control
Suite \citep{stone2021distracting}, which adds visual background, color, and
camera variation to pixel observations, our benchmark adds a latent
state-vector process and then mixes it across all observed coordinates.  The
flattened native observation dimensions are respectively
$d_x\in\{5,6,17,24\}$.  The native observations provide real nonlinear
rigid-body dynamics, contacts, joint limits, and domain-specific geometry; no
constructed transition equation is used.

To isolate the common-mode question, each environment is augmented with an
autonomous process $c_t\in\R^{64}$:
\begin{equation}
  c_{t+1}=0.97c_t+0.35\varepsilon_t,
  \qquad \varepsilon_t\sim\mathcal N(0,I).
  \label{eq:dm-common-dynamics}
\end{equation}
Let $x_t\in\R^{d_x}$ be the flattened native task observation.  The model sees
\begin{equation}
  o_t
  =
  Q
  \begin{bmatrix}
    6c_t\\x_t
  \end{bmatrix},
  \label{eq:dm-mixed-observation}
\end{equation}
where $Q$ is a fixed random orthogonal matrix.  Thus every observed coordinate
mixes action-independent and controllable information; a model cannot remove
the nuisance by dropping a known block of coordinates.

\subsection{Action prototypes and paired branches}

DM Control actions are continuous, whereas the current \CQM{} implementation
centers a finite action set.  Cartpole uses three evenly spaced scalar controls
$\{-1,0,1\}$.  The other domains use five fixed prototypes: zero control and
two randomly sampled directions together with their antithetic negatives,
scaled to the action bounds.  These prototypes define the intervention set;
all reported ``optimal actions'' are optimal only within this finite set.

An initial physics state is obtained by resetting the environment and applying
between zero and 20 random warm-up controls.  From that state, we clone one
branch per candidate first action and roll each branch for $H=12$ steps with
$\gamma=0.95$.  The continuation actions after the first step are sampled
uniformly and shared across branches.  The complete exogenous-noise sequence
is also shared.  Consequently the branch family satisfies the coupling in
\cref{thm:cancellation,prop:paired-variance}: only the first control
intervention differs.

\subsection{Held-out reward queries}

The experiment evaluates transfer across linear reward queries, not the
environments' built-in rewards.  For each domain we sample unit vectors
$g\in\R^{d_x}$ and define
\begin{equation}
  r_g(t)=g^\top x_{t+1}.
  \label{eq:dm-linear-reward}
\end{equation}
Because $Q$ is orthogonal, $w_g=Q[0;g]$ satisfies
$r_g(t)=w_g^\top o_{t+1}$.  The discounted branch feature
\begin{equation}
  F_H(s,a)=\E\!\left[
    \sum_{k=0}^{H-1}\gamma^k o_{t+k+1}
    \,\middle|\,s_t=s,\doop(a_t=a),\pi
  \right]
  \label{eq:dm-branch-feature}
\end{equation}
therefore answers every task through $w_g^\top F_H(s,a)$.  We use 32 reward
directions to train reward-aware baselines and reserve 16 independently
sampled directions for evaluation.  \CQM{}, the world model, and successor
features receive no reward vector in their representation losses.

\begin{remark}[What the benchmark does and does not measure]
The names Swingup, Easy, Run, and Walk identify the underlying physics and
observation maps.  Held-out accuracy measures the best first action for an
unseen linear reward under the fixed uniform continuation policy.  It is
neither native-task return nor long-horizon closed-loop success.  This choice
matches the finite-feature theory in \cref{eq:q-readout} and cleanly tests
reward transfer, but it should not be read as evidence that the learned model
solves the standard DM Control objective.
\end{remark}

\subsection{Models and protocol}

We evaluate the branched finite-feature estimator \CQM-B.  The oracle quotient
uses exact paired effects and is an unattainable ceiling.  Learned comparisons
are: (i) an absolute one-step world model rolled forward with mean continuation
actions; (ii) successor features trained by TD under the uniform continuation
policy; (iii) a direct task-conditioned Monte Carlo value predictor, used as a
value-equivalence proxy; (iv) Task-TD, a universal value function approximator
with target-network policy-evaluation backups; and (v) a COCOA-style
contribution estimator with a distilled task-conditioned policy
\citep{grimm2020value,meulemans2023cocoa}.  These reward-aware comparisons
are controlled proxies chosen to expose different prediction targets; they
are not exact reproductions of every algorithmic detail in the cited work.

\CQM-B, the world model, and successor features use identical architectures
within each domain: two width-128 hidden layers, a 16-dimensional action
embedding, an eight-dimensional bottleneck, and a linear decoder to the mixed
observation dimension.  Thus differences among these three models reflect
their prediction targets rather than parameter budget.  Parameter counts vary
slightly with $d_x$ (approximately 79--82k for each vector model).

Each seed contains 20,000 training and 4,000 test initial states.  All methods
use AdamW, learning rate $3\times10^{-4}$, weight decay $10^{-5}$, batch size
256, gradient clipping at 10, and 4,000 updates.  Target networks update every
100 steps.  COCOA policy distillation uses 2,000 updates and temperature
$0.25$.  We report seeds 7--11; confidence intervals are $1.96$ standard
errors across the five runs.

\subsection{Metrics}

For each test state $s$ and reward direction $g$, let
$a^\star_g(s)=\argmax_a Q_g(s,a)$ be the best prototype according to the
paired simulator return, and let
$\widehat a_g(s)=\argmax_a\widehat Q_g(s,a)$.  Action accuracy is
\begin{equation}
  \operatorname{Acc}
  =
  \Pr_{s,g}\!\left[\widehat a_g(s)=a^\star_g(s)\right].
  \label{eq:action-accuracy}
\end{equation}
Normalized regret retains the cost of near misses:
\begin{equation}
  \operatorname{NReg}
  =
  \frac{
    \E_{s,g}\!\left[
      Q_g(s,a^\star_g)-Q_g(s,\widehat a_g)
    \right]
  }{
    \operatorname{Std}_{s,a,g}[Q_g(s,a)]
  }.
  \label{eq:normalized-regret}
\end{equation}
Accuracy is strict---a different argmax counts as an error even when two
actions have nearly equal return---so regret is needed to measure severity.
For vector predictors, effect NMSE is
\begin{equation}
  \operatorname{NMSE}_{\mathrm{eff}}
  =
  \frac{
    \E\|\widehat\kappa(s,a)-\kappa(s,a)\|_2^2
  }{
    \E\|\kappa(s,a)\|_2^2
  }.
  \label{eq:effect-nmse}
\end{equation}
For the score-alignment diagnostic, we remove each model's arbitrary
action-independent baseline and scale:
\begin{equation}
  \widetilde A_g(s,a)
  =
  \frac{
    \widehat Q_g(s,a)-|\cA|^{-1}\sum_b\widehat Q_g(s,b)
  }{
    \left[
      |\cA|^{-1}\sum_b
      \left(
        \widehat Q_g(s,b)-|\cA|^{-1}\sum_c\widehat Q_g(s,c)
      \right)^2
    \right]^{1/2}
  }.
  \label{eq:normalized-advantage}
\end{equation}
This normalization tests the geometry of relative action scores rather than
their arbitrary offset or overall scale.

\subsection{Aggregate results}

\begin{table}[t]
  \caption{Held-out action accuracy over five seeds (mean $\pm$ approximate
  95\% confidence interval).  Bold is the best learned method.  Chance is
  $1/3$ on Cartpole and $1/5$ elsewhere.}
  \label{tab:dm-accuracy}
  \centering
  \small
  \resizebox{\linewidth}{!}{%
  \begin{tabular}{lcccccc}
    \toprule
    Domain & \CQM{}-B & World & SF & Direct value & Task-TD & COCOA\\
    \midrule
    Cartpole
      & $\mathbf{0.9998\pm0.0003}$ & $0.2521\pm0.1389$
      & $0.4236\pm0.0840$ & $0.9904\pm0.0084$
      & $0.8407\pm0.0766$ & $0.6445\pm0.0791$\\
    Reacher
      & $\mathbf{0.9553\pm0.0290}$ & $0.2598\pm0.0592$
      & $0.3456\pm0.0647$ & $0.8534\pm0.0974$
      & $0.6769\pm0.0892$ & $0.6316\pm0.1331$\\
    Cheetah
      & $\mathbf{0.6411\pm0.0118}$ & $0.2693\pm0.0901$
      & $0.3713\pm0.0362$ & $0.6115\pm0.0523$
      & $0.5338\pm0.0501$ & $0.5957\pm0.0398$\\
    Walker
      & $\mathbf{0.3134\pm0.0101}$ & $0.2989\pm0.0114$
      & $0.2719\pm0.0099$ & $0.2665\pm0.0375$
      & $0.2900\pm0.0171$ & $0.2268\pm0.0357$\\
    \bottomrule
  \end{tabular}}
\end{table}

\begin{table}[t]
  \caption{Held-out normalized regret (lower is better).  The oracle has zero
  regret.  Bold is the best learned method.}
  \label{tab:dm-regret}
  \centering
  \small
  \resizebox{\linewidth}{!}{%
  \begin{tabular}{lcccccc}
    \toprule
    Domain & \CQM{}-B & World & SF & Direct value & Task-TD & COCOA\\
    \midrule
    Cartpole
      & $\mathbf{0.0000001\pm0.0000002}$ & $0.1909\pm0.0598$
      & $0.1412\pm0.0208$ & $0.00007\pm0.00007$
      & $0.0139\pm0.0086$ & $0.0850\pm0.0310$\\
    Reacher
      & $\mathbf{0.0011\pm0.0005}$ & $0.1561\pm0.0262$
      & $0.1095\pm0.0349$ & $0.0039\pm0.0027$
      & $0.0167\pm0.0082$ & $0.0192\pm0.0094$\\
    Cheetah
      & $\mathbf{0.0542\pm0.0041}$ & $0.2044\pm0.0327$
      & $0.1450\pm0.0121$ & $0.0606\pm0.0143$
      & $0.0894\pm0.0216$ & $0.0627\pm0.0078$\\
    Walker
      & $\mathbf{0.2427\pm0.0209}$ & $0.2536\pm0.0199$
      & $0.2716\pm0.0182$ & $0.2847\pm0.0332$
      & $0.2604\pm0.0206$ & $0.3125\pm0.0099$\\
    \bottomrule
  \end{tabular}}
\end{table}

\begin{table}[t]
  \caption{Effect NMSE for the three matched vector predictors.  Unlike
  absolute-future NMSE, this metric isolates action-dependent prediction
  error.}
  \label{tab:dm-effect}
  \centering
  \small
  \begin{tabular}{lccc}
    \toprule
    Domain & \CQM{}-B & World & SF\\
    \midrule
    Cartpole & $\mathbf{0.000007}$ & $1.2182$ & $1.0658$\\
    Reacher  & $\mathbf{0.0385}$   & $1.0369$ & $1.0125$\\
    Cheetah  & $\mathbf{0.4503}$   & $1.0228$ & $1.0113$\\
    Walker   & $\mathbf{1.0453}$   & $1.6460$ & $1.1585$\\
    \bottomrule
  \end{tabular}
\end{table}

\paragraph{Cartpole and Reacher show the cleanest quotient advantage.}
On Cartpole, \CQM{} is essentially exact and reduces normalized regret by
roughly six orders of magnitude relative to the absolute world model.  On
Reacher, it reaches $95.53\%$ held-out accuracy, compared with $25.98\%$ for
the world model and $34.56\%$ for successor features.  The matched
architectures rule out parameter count as the explanation.  Instead,
\cref{tab:dm-effect} shows that direct quotient supervision estimates the
action effect itself, whereas the absolute predictors attain effect NMSE near
or above one despite absolute-future NMSE near $0.92$--$0.93$.

\paragraph{Cheetah remains favorable but is materially harder.}
The native state dimension increases to 17 and the same rank-eight bottleneck
must represent more varied action effects.  \CQM{} still obtains the highest
accuracy ($64.11\%$) and lowest regret ($0.0542$), but the direct value and
COCOA baselines are close.  Its effect NMSE rises to $0.4503$, consistent with
the feature-level regret link rather than with a domain-independent guarantee
of near-perfect recovery.

\paragraph{Walker exposes the present capacity limit.}
Walker has a 24-dimensional native observation and contact-rich dynamics.
All methods perform poorly: \CQM{} reaches only $31.34\%$ accuracy, versus a
$20\%$ random-action level, and its effect NMSE exceeds one.  Although it has
the best mean accuracy and regret, the margin over the world model is small
and the confidence intervals overlap.  We therefore treat Walker as a
negative or stress-test result: quotient targeting alone does not compensate
for an underspecified shared rank-eight decoder; the pointwise
$|\cA|-1$ rank bound does not prevent the action-effect subspace from rotating
across states.

\paragraph{Reward-free effects versus reward-aware values.}
Direct value prediction is highly competitive on Cartpole, Reacher, and
Cheetah because it learns scalar decision queries directly.  This is
consistent with \cref{prop:dual-recovery}: 32 training directions can provide
broad coverage of these finite physical feature spaces.  \CQM{} nevertheless
matches or improves on these reward-aware methods while its representation
loss never observes a reward direction.  The useful claim is therefore
reusability of a primal action-effect representation, not that reward-aware
critics are incapable of representing the same decisions.

\subsection{Reacher score geometry}

\begin{figure*}[t]
  \centering
  \includegraphics[width=\textwidth]{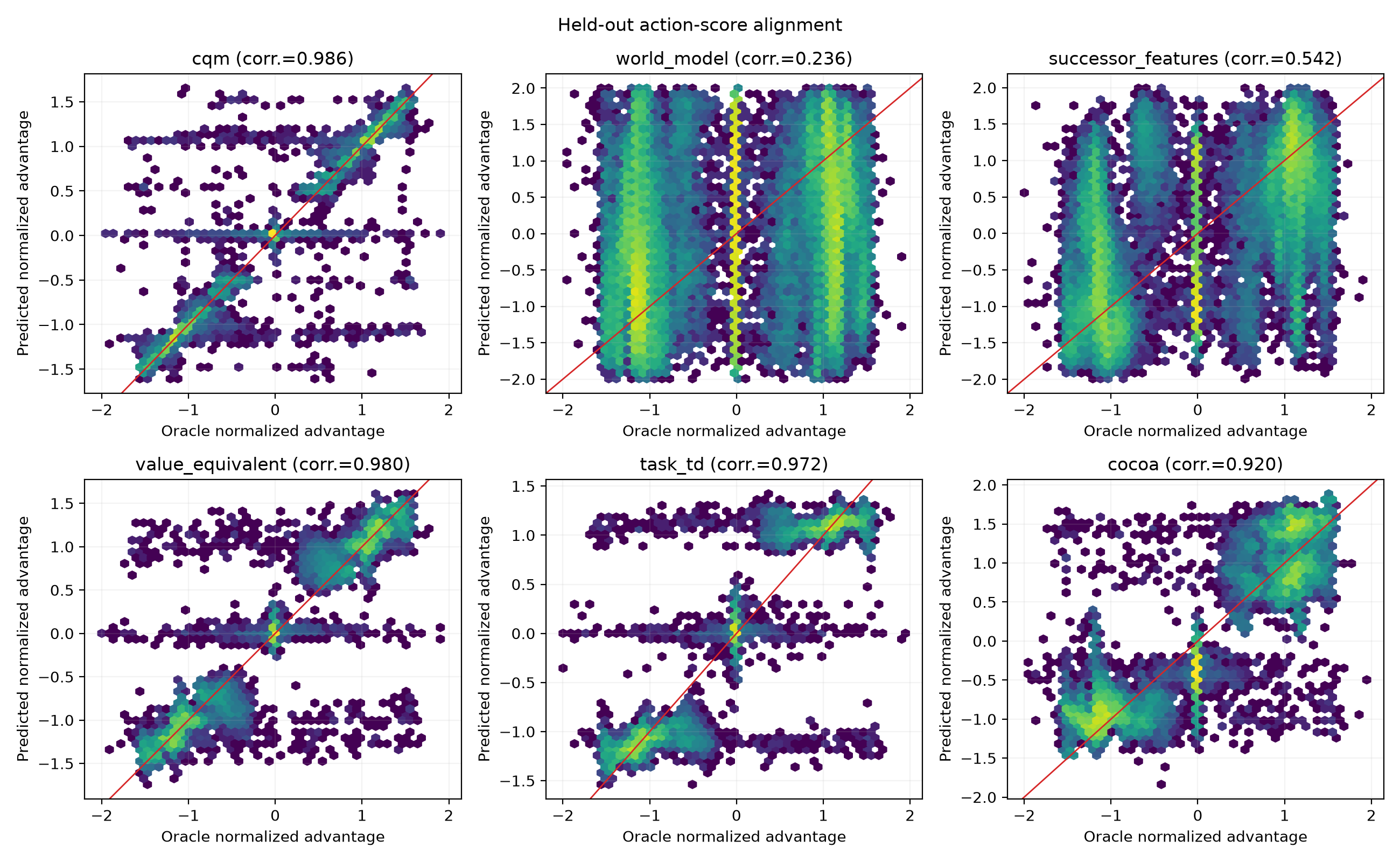}
  \caption{Held-out action-score alignment for one representative Reacher
  seed.  Each hexagon
  aggregates $(s,a,g)$ triples.  The horizontal coordinate is the oracle
  advantage and the vertical coordinate is the model advantage, each centered
  over actions and normalized within $(s,g)$.  The red diagonal is perfect
  agreement.  \CQM{} concentrates tightly around the diagonal
  (correlation $0.986$), while the absolute world model ($0.236$) and
  successor features ($0.542$) show weak or fragmented action geometry.
  Reward-aware methods are strong but less tightly aligned.}
  \label{fig:reacher-alignment}
\end{figure*}

\begin{figure}[t]
  \centering
  \includegraphics[width=\linewidth]{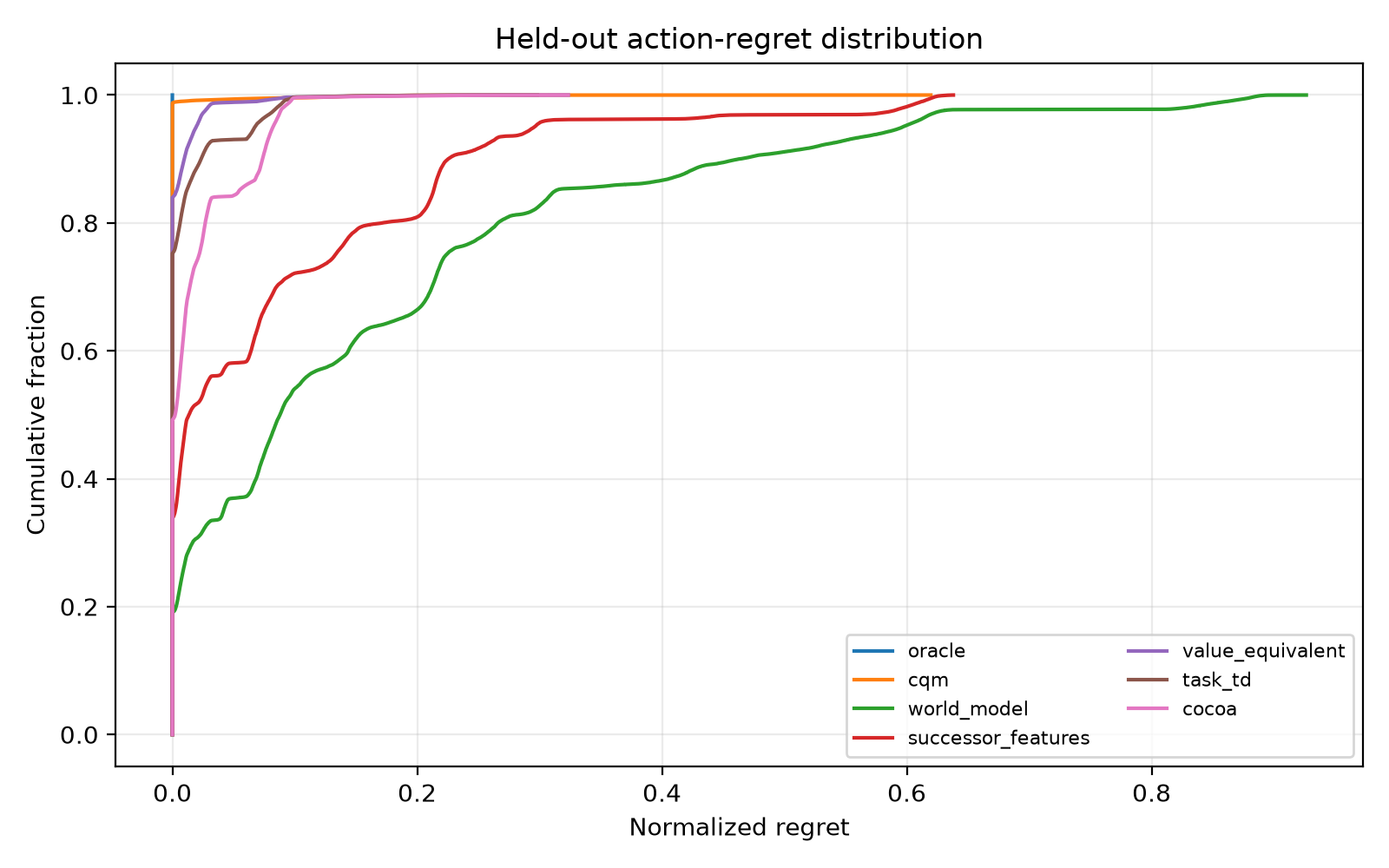}
  \caption{CDF of held-out normalized regret for the same Reacher run.  Curves
  farther left and higher are better.  \CQM{} places almost all queries at
  zero or negligible regret; the world model has a long tail, and successor
  features improves on the world model but remains substantially worse than
  the quotient and reward-aware critics.}
  \label{fig:reacher-regret}
\end{figure}

\Cref{fig:reacher-alignment} explains why action accuracy and effect quality
are related but not identical.  Accuracy uses only the largest score, whereas
the alignment plot tests the entire centered action-score vector.  A model can
occasionally select the correct action while badly distorting the remaining
ordering; conversely, a near miss can count as inaccurate while incurring
little regret.  \Cref{fig:reacher-regret} therefore complements the strict
accuracy statistic by showing the full distribution of decision costs.

\section{Theoretical Properties}

We now formalize the properties exercised by the preceding experiments.  This
section characterizes what information the quotient retains and removes and
then separates exact representation statements from finite-feature,
finite-action, and estimation errors.  All statements apply pointwise in $s$
and extend to measurable function spaces under the corresponding
integrability assumptions.

\paragraph{Roadmap.}
There are three logically distinct questions.  First, is the quotient the
right \emph{information object}?  \Cref{thm:duality,thm:gauge} show that it
preserves exactly all bounded-reward action comparisons and removes only an
action-independent baseline.  Second, can that object be \emph{identified}?
\Cref{thm:pairwise,thm:projection} show that pairwise branch differences
identify a unique centered representative.  Third, what happens after
\emph{approximation}?  The common-mode cancellation result
(\Cref{thm:cancellation}) explains why paired targets can be statistically
easier; the action-effect rank result (\Cref{thm:rank}) explains when a
low-rank decoder is plausible; and the regret and reward-projection results
(Theorem~\ref{thm:regret} and Proposition~\ref{prop:reward-error}) explain how
model and representation errors affect decisions.  Keeping these questions
separate prevents two common confusions: measure-level sufficiency does not
imply that every finite feature map is sufficient, and target cancellation
does not imply that a finite network estimates the remaining action effect
without error.

\paragraph{Why centered values suffice.}
For fixed $s$ and reward $r$, write
\[
  \bar Q^\pi_{r,\rho,H}(s)
  =
  \sum_b\rho(b\mid s)Q^\pi_{r,H}(s,b).
\]
This term depends on the state and reward but not on the candidate action.
Consequently,
\[
  \argmax_a Q^\pi_{r,H}(s,a)
  =
  \argmax_a\!\left[
    Q^\pi_{r,H}(s,a)-\bar Q^\pi_{r,\rho,H}(s)
  \right].
\]
The quotient therefore discards an additive baseline that is provably
irrelevant to the current action ordering; it does not discard any pairwise
contrast.

\begin{theorem}[Effect--value duality]
\label{thm:duality}
For any bounded measurable reward $r:\cZ\to\R$, continuation policy $\pi$,
horizon $H$, and reference distribution $\rho$,
\begin{equation}
  Q^\pi_{r,H}(s,a)
  -
  \sum_b\rho(b\mid s)Q^\pi_{r,H}(s,b)
  =
  \int_{\cZ}r(z)\,\mathrm dK^\pi_{H,s,a}(z).
  \label{eq:duality}
\end{equation}
Consequently, the signed measure $K^\pi_{H,s,a}$ is a reward-independent
advantage generator.  For $r_w=w^\top\phi$, the right-hand side reduces to
$w^\top\kappa^\pi_H(s,a)$.
\end{theorem}

\begin{proof}
By \cref{eq:measure-value}, each action value is the integral of the same
reward against the corresponding successor measure.  Because the action set
is finite, the reference-weighted sum can be moved inside the integral:
\begin{align*}
  Q^\pi_{r,H}(s,a)
  -\sum_b\rho(b\mid s)Q^\pi_{r,H}(s,b)
  &=
  \int r\,\mathrm dM^\pi_{H,s,a}
  -\sum_b\rho(b\mid s)\int r\,\mathrm dM^\pi_{H,s,b}\\
  &=
  \int r\,\mathrm d\!\left(
    M^\pi_{H,s,a}
    -\sum_b\rho(b\mid s)M^\pi_{H,s,b}
  \right)\\
  &=
  \int r\,\mathrm d\!\left(
    M^\pi_{H,s,a}-\overline M^\pi_{H,s}
  \right)\\
  &=\int r\,\mathrm dK^\pi_{H,s,a}.
\end{align*}
The penultimate equality is the definition of the reference measure, and the
last is the definition of the effect measure.  For $r_w=w^\top\phi$, linearity
holds coordinate by coordinate, so
\[
  \int w^\top\phi(z)\,\mathrm dK(z)
  =
  w^\top\!\left(\int\phi(z)\,\mathrm dK(z)\right)
  =
  w^\top\kappa(s,a),
\]
where the final equality is \cref{eq:feature-projection}.
\end{proof}

\begin{proposition}[Zero mass and reference centering]
\label{prop:zero-mass}
For every $s,a$,
$K^\pi_{H,s,a}(\cZ)=0$, and
$\sum_a\rho(a\mid s)K^\pi_{H,s,a}=0$ as a signed measure.  Consequently,
$\sum_a\rho(a\mid s)\int r\,\mathrm dK^\pi_{H,s,a}=0$ for every bounded
measurable $r$.
\end{proposition}

\begin{proof}
For every first action, the successor measure places total discounted mass
$m_H=\sum_{k=0}^{H-1}\gamma^k$ on $\cZ$.  Therefore
\[
  \overline M^\pi_{H,s}(\cZ)
  =
  \sum_b\rho(b\mid s)M^\pi_{H,s,b}(\cZ)
  =
  m_H\sum_b\rho(b\mid s)
  =
  m_H.
\]
Subtracting the two masses gives
$K^\pi_{H,s,a}(\cZ)=m_H-m_H=0$.  For reference centering, expand the
definition of $K$ and keep the normalization of $\rho$ explicit:
\[
  \sum_a\rho(a\mid s)K^\pi_{H,s,a}
  =
  \sum_a\rho(a\mid s)M^\pi_{H,s,a}
  -
  \left(\sum_a\rho(a\mid s)\right)\overline M^\pi_{H,s}
  =
  \overline M^\pi_{H,s}-\overline M^\pi_{H,s}
  =
  0.
\]
This is equality of signed measures.  Integrating any bounded measurable
reward against both sides and using linearity gives the final claim.
\end{proof}

\begin{theorem}[Measure gauge invariance and completeness]
\label{thm:gauge}
For two action-indexed finite successor-measure families
$\mathcal M_1,\mathcal M_2\in\mathfrak F$, the following are equivalent:
\begin{enumerate}
  \item $\mathcal M_1-\mathcal M_2\in\mathfrak B$;
  \item their $\rho$-centered signed effect measures are equal;
  \item for every $s,a,b$ and bounded measurable $r$,
  \[
    \int r\,\mathrm d(M^1_{s,a}-M^1_{s,b})
    =
    \int r\,\mathrm d(M^2_{s,a}-M^2_{s,b}).
  \]
\end{enumerate}
Thus, the quotient is invariant to all action-independent future additions,
and it identifies no two models that can be distinguished by any bounded
reward action comparison.
\end{theorem}

\begin{proof}
We prove the cycle of implications.

\emph{$1\Rightarrow2$.}  Membership in $\mathfrak B$ means that for each
state there is a signed measure $C_s$ satisfying
$M^1_{s,a}-M^2_{s,a}=C_s$ for every action.  Centering this difference gives
\[
  C_s-\sum_b\rho(b\mid s)C_s
  =
  C_s-C_s
  =
  0,
\]
so the two centered families coincide.

\emph{$2\Rightarrow3$.}  Write the common centered family as $K$.  Centering
subtracts the same reference measure from every action, hence
$M^i_{s,a}-M^i_{s,b}=K^i_{s,a}-K^i_{s,b}$ for each model $i$.  Equality of
the centered families therefore implies equality of the two pairwise signed
measures.  Integrating the bounded reward $r$ against them gives statement~3.

\emph{$3\Rightarrow1$.}  Fix $s,a,b$ and define
\[
  \Delta_{s,a,b}
  =
  (M^1_{s,a}-M^1_{s,b})
  -(M^2_{s,a}-M^2_{s,b}).
\]
Statement~3 says that $\int r\,\mathrm d\Delta_{s,a,b}=0$ for every bounded
measurable $r$.  In particular, taking $r$ to be the indicator of any
measurable set shows that $\Delta_{s,a,b}$ assigns zero mass to every such
set; therefore it is the zero signed measure.  Fix an arbitrary reference
action $b_0$ and set $C_s=M^1_{s,b_0}-M^2_{s,b_0}$.  Rearranging
$\Delta_{s,a,b_0}=0$ yields
$M^1_{s,a}-M^2_{s,a}=C_s$ for every $a$.  The difference is thus
action-independent, which is exactly statement~1.
\end{proof}

\begin{corollary}[Minimal decision sufficiency]
\label{cor:minimal}
Any representation sufficient to recover pairwise action values for every
bounded measurable reward must distinguish every pair of classes in
$\mathfrak F/\mathfrak B$.  Therefore the signed effect measure $K$ is the
coarsest exact representation for this reward query family.
\end{corollary}

\begin{proof}
Let $T$ be any representation from which every bounded-reward pairwise action
value can be recovered.  Suppose, for contradiction, that $T$ maps two
distinct quotient classes to the same representation.  Since the classes are
distinct, statement~1 of \cref{thm:gauge} is false.  By the equivalence in that
theorem, statement~3 is also false: there must be a state, two actions, and a
bounded measurable reward for which the two models yield different pairwise
values.  A decoder receiving the same value of $T$ for both models cannot
return two different answers to that query.  This contradicts the assumed
sufficiency of $T$.  Thus every sufficient representation must separate all
quotient classes.  Since $K$ labels exactly those classes by
\cref{thm:gauge}, it is the coarsest exact representation.
\end{proof}

\begin{theorem}[Identification by pairwise learning]
\label{thm:pairwise}
Fix $s$ and let $Y_a\in\R^d$ be an unbiased branched-rollout target with
$\E[Y_a\mid s]=F^\pi_H(s,a)$.  Let $G_s$ be a connected comparison graph on
the actions, and consider the population objective
\begin{equation}
  \mathcal L_s(f)
  =
  \E_{(a,b)\sim G_s}
  \E\left[
    \|f(s,a)-f(s,b)-(Y_a-Y_b)\|_2^2
    \,\middle|\,s,a,b
  \right].
  \label{eq:pairwise-loss}
\end{equation}
If the function class can realize the conditional mean differences, every
population minimizer satisfies
\[
  f^\star(s,a)=F^\pi_H(s,a)+c(s)
\]
for an action-independent vector $c(s)$.  Consequently,
$\mathcal C_\rho f^\star=\kappa^\pi_H$.
\end{theorem}

\begin{proof}
Fix an edge $(a,b)$ and abbreviate
$D_{ab}=Y_a-Y_b$ and $u_{ab}=f(s,a)-f(s,b)$.  Conditional on the edge, the
squared-loss decomposition is
\[
  \E[\|u_{ab}-D_{ab}\|_2^2\mid s,a,b]
  =
  \|u_{ab}-\E[D_{ab}\mid s,a,b]\|_2^2
  +
  \E[\|D_{ab}-\E D_{ab}\|_2^2\mid s,a,b].
\]
The second term does not depend on $f$.  Hence every realizable population
minimum must set the first term to zero on every edge with positive sampling
probability:
\[
  f(s,a)-f(s,b)
  =
  \E[Y_a-Y_b\mid s,a,b]
  =
  F^\pi_H(s,a)-F^\pi_H(s,b).
\]
The last equality uses unbiasedness of both branch targets.  Choose a
reference action $a_0$.  For any $a$, connectivity supplies a path
$a_0=v_0,v_1,\ldots,v_m=a$.  Sum the edge equalities along this path.  Both
sides telescope, giving
\[
  f(s,a)-f(s,a_0)
  =
  F^\pi_H(s,a)-F^\pi_H(s,a_0).
\]
After rearrangement,
$f(s,a)=F^\pi_H(s,a)+c(s)$ with
$c(s)=f(s,a_0)-F^\pi_H(s,a_0)$, and the same $c(s)$ works for every action.
Finally,
$\mathcal C_\rho c(s)=c(s)-\sum_b\rho(b\mid s)c(s)=0$; centering therefore
removes the only unidentified component and returns $\kappa^\pi_H$.
\end{proof}

\begin{theorem}[Canonical minimum-norm representative]
\label{thm:projection}
At the finite feature level, equip action-indexed vectors at state $s$ with
the weighted inner product
\[
  \langle F,G\rangle_{\rho,s}
  =
  \sum_a\rho(a\mid s)F(s,a)^\top G(s,a).
\]
Then $\mathcal C_\rho$ is the orthogonal projection onto the zero-mean
subspace
$\{K:\sum_a\rho(a\mid s)K(s,a)=0\}$.
Moreover, $\mathcal C_\rho F$ is the unique minimum-norm element of $[F]$.
\end{theorem}

\begin{proof}
Write $\bar F_\rho(s)=\sum_a\rho(a\mid s)F(s,a)$.  By construction,
$\mathcal C_\rho F=F-\bar F_\rho$ has zero weighted mean because
\[
  \sum_a\rho(a\mid s)(\mathcal C_\rho F)(s,a)
  =
  \bar F_\rho(s)
  -\left(\sum_a\rho(a\mid s)\right)\bar F_\rho(s)
  =
  0.
\]
For any zero-mean $K$ and action-independent $B(s,a)=b(s)$,
\[
  \langle K,B\rangle_{\rho,s}
  =
  \left(\sum_a\rho(a\mid s)K(s,a)\right)^\top b(s)=0.
\]
Thus the two subspaces are orthogonal.  They also span the whole space, since
every $F$ has the explicit decomposition
\[
  F=\mathcal C_\rho F+\bar F_\rho,
\]
whose first term is zero-mean and whose second is action-independent.
Their intersection contains only zero: an action-independent $B=b(s)$ with
zero weighted mean satisfies $b(s)\sum_a\rho(a\mid s)=b(s)=0$.  They are
therefore orthogonal complements, and $\mathcal C_\rho$ is the orthogonal
projection onto the zero-mean subspace.

Every representative in $[F]$ differs from $\mathcal C_\rho F$ by an
action-independent $B$.  Orthogonality eliminates the cross term, so
\[
  \|\mathcal C_\rho F+B\|_{\rho,s}^2
  =
  \|\mathcal C_\rho F\|_{\rho,s}^2+\|B\|_{\rho,s}^2.
\]
The second term is nonnegative and is zero only when $B=0$, proving both
minimality and uniqueness.
\end{proof}

\begin{remark}[Dependence on the reference distribution]
The quotient class does not depend on $\rho$, but its coordinates do.  For
two full-support references $\rho$ and $\rho'$,
\[
  \mathcal C_{\rho'}F(s,a)-\mathcal C_\rho F(s,a)
  =
  \bar F_\rho(s)-\bar F_{\rho'}(s),
\]
which is constant across $a$.  Hence both representatives give identical
pairwise differences and action rankings for every reward.  Uniform $\rho$ is
convenient for a symmetric finite action set; a behavior-matched $\rho$ may be
preferable when some actions are sampled more often.
\end{remark}

\begin{remark}[Why direct quotient fitting differs at finite capacity]
If an absolute model class could represent $F^\pi_H$ exactly, centering its
prediction would recover $\kappa^\pi_H$.  The distinction arises before that
limit.  An absolute loss allocates its finite decoder rank and gradient signal
to both $\mathcal C_\rho F$ and the orthogonal action-independent component
$(I-\mathcal C_\rho)F$.  Direct \CQM{} fitting presents only the former to the
approximator.  Because projection of the target and projection onto a
restricted nonlinear model class need not commute, post-hoc centering does
not in general undo capacity spent fitting the common component.
\end{remark}

\begin{theorem}[Exact common-mode cancellation]
\label{thm:cancellation}
Suppose synchronized branch outcomes admit the decomposition
\begin{equation}
  \widehat F(s,a)=C(s,\xi)+E(s,a,\xi),
  \label{eq:noise-decomp}
\end{equation}
where $C$ can be arbitrarily high-dimensional or high-variance but is shared
by every first-action branch under common random numbers, a classical
variance-reduction coupling \citep{glasserman2004monte,glynn2002common}.  Then
\begin{equation}
  \mathcal C_\rho\widehat F(s,a)
  =
  E(s,a,\xi)-\sum_b\rho(b\mid s)E(s,b,\xi),
  \label{eq:cancelled}
\end{equation}
which is independent of $C$.  In particular, the variance of $C$ contributes
nothing to the paired quotient target.
\end{theorem}

\begin{proof}
Apply the centering operator in \cref{eq:center} to the decomposition in
\cref{eq:noise-decomp}.  Since the same realization $C(s,\xi)$ appears in
every branch,
\begin{align*}
  \mathcal C_\rho\widehat F(s,a)
  &=
  C+E(s,a,\xi)
  -\sum_b\rho(b\mid s)\bigl(C+E(s,b,\xi)\bigr)\\
  &=
  C+E(s,a,\xi)
  -C\sum_b\rho(b\mid s)
  -\sum_b\rho(b\mid s)E(s,b,\xi)\\
  &=
  E(s,a,\xi)-\sum_b\rho(b\mid s)E(s,b,\xi),
\end{align*}
because $\sum_b\rho(b\mid s)=1$.  The final expression contains no occurrence
of $C$ for any realization of $\xi$, so cancellation is sample-wise rather
than only in expectation.  A random variable absent from the target cannot
contribute to its variance, which proves the final statement.
\end{proof}

\begin{remark}
The theorem is stronger than merely stating that common modes do not affect
the optimal action.  It says they disappear from each stochastic training
target before function approximation.  If branches use independent noise,
centering still removes their common expectation but not sample-level noise;
the common-random-number coupling is therefore an algorithmically important
variance reduction device.
\end{remark}

\begin{proposition}[Variance removed by paired branches]
\label{prop:paired-variance}
Consider two actions and scalar branch outcomes
$Y_a=C+E_a$ and $Y_b=C+E_b$, where the same random variable $C$ is reused
across branches.  Then
\begin{equation}
  \operatorname{Var}(Y_a-Y_b)
  =
  \operatorname{Var}(E_a-E_b).
  \label{eq:paired-variance}
\end{equation}
If instead the branches use independent copies $C_a,C_b$ that are independent
of $(E_a,E_b)$, then
\begin{equation}
  \operatorname{Var}\!\left[
    (C_a+E_a)-(C_b+E_b)
  \right]
  =
  2\operatorname{Var}(C)
  +
  \operatorname{Var}(E_a-E_b).
  \label{eq:independent-variance}
\end{equation}
\end{proposition}

\begin{proof}
With shared noise,
\[
  Y_a-Y_b=(C+E_a)-(C+E_b)=E_a-E_b,
\]
so taking variances proves \cref{eq:paired-variance}.  With independent
copies, write the difference as
$D_C+D_E$, where $D_C=C_a-C_b$ and $D_E=E_a-E_b$.  The general sum formula
gives
\[
  \operatorname{Var}(D_C+D_E)
  =
  \operatorname{Var}(D_C)
  +\operatorname{Var}(D_E)
  +2\operatorname{Cov}(D_C,D_E).
\]
Independence of $(C_a,C_b)$ from $(E_a,E_b)$ makes the covariance term zero.
Moreover,
\begin{align*}
  \operatorname{Var}(D_C)
  &=
  \operatorname{Var}(C_a)
  +\operatorname{Var}(C_b)
  -2\operatorname{Cov}(C_a,C_b)\\
  &=
  \operatorname{Var}(C)+\operatorname{Var}(C)-0
  =
  2\operatorname{Var}(C),
\end{align*}
using identical distribution and independence of the two copies.  Substitution
gives \cref{eq:independent-variance}.
\end{proof}

\Cref{prop:paired-variance} gives a quantitative interpretation of
\cref{thm:cancellation}.  The gain is not merely that an unbiased estimator
has been centered: the entire variance contribution of a shared nuisance is
absent from every supervised action difference.  The conclusion weakens when
the nuisance is only approximately shared, in which case the residual is the
branch-to-branch mismatch in that nuisance.

\begin{theorem}[Action-effect rank]
\label{thm:rank}
For fixed $s$, place the effect vectors in the matrix
$K_s=[\kappa(s,a_1),\ldots,\kappa(s,a_{|\cA|})]\in\R^{d\times|\cA|}$.
Under uniform $\rho$,
\begin{equation}
  \rank(K_s)\leq |\cA|-1.
  \label{eq:rank-actions}
\end{equation}
If all pairwise action effects lie in an $r$-dimensional subspace
$U_s\subseteq\R^d$, then
\begin{equation}
  \rank(K_s)\leq\min\{r,|\cA|-1\}.
  \label{eq:rank-control}
\end{equation}
\end{theorem}

\begin{proof}
Under the uniform reference, the centering constraint is
\[
  \sum_{j=1}^{|\cA|}\kappa(s,a_j)=0,
\]
which in matrix notation is $K_s\1=0$.  The all-ones vector is nonzero, so the
null space of $K_s$ has dimension at least one.  Rank--nullity for a matrix
with $|\cA|$ columns then yields
$\rank(K_s)\leq|\cA|-1$.

For the second claim, fix any action $a_i$.  Its centered effect can be
written explicitly as
\[
  \kappa(s,a_i)
  =
  F(s,a_i)-\sum_j\rho(a_j\mid s)F(s,a_j)
  =
  \sum_j\rho(a_j\mid s)
  \bigl(F(s,a_i)-F(s,a_j)\bigr).
\]
Every vector in the final sum lies in $U_s$ by assumption.  Since $U_s$ is a
linear subspace, every column of $K_s$ lies in $U_s$.  Hence
$\operatorname{col}(K_s)\subseteq U_s$ and $\rank(K_s)\leq r$.  Combining
this bound with the first gives the stated minimum.
\end{proof}

\begin{remark}[Pointwise rank versus a shared decoder]
\Cref{thm:rank} is pointwise in $s$.  It does not imply that one fixed
$(|\cA|-1)$-dimensional output subspace contains the effects at every state:
the subspace $U_s$ may rotate with $s$, and
$\operatorname{span}(\bigcup_s U_s)$ can be much larger.  The implemented
linear decoder is shared across states, so its required rank is governed by
this global span.  This distinction is especially relevant for contact-rich
Walker dynamics.
\end{remark}

\begin{proposition}[Finite action-prototype error]
\label{prop:prototype-error}
Let $\widetilde{\cA}$ be a compact continuous action space and
$\cA_\epsilon\subset\widetilde{\cA}$ an $\epsilon$-cover.  If
$Q_r(s,\cdot)$ is $L_r$-Lipschitz, then
\begin{equation}
  0
  \leq
  \max_{a\in\widetilde{\cA}}Q_r(s,a)
  -
  \max_{\tilde a\in\cA_\epsilon}Q_r(s,\tilde a)
  \leq L_r\epsilon.
  \label{eq:prototype-error}
\end{equation}
\end{proposition}

\begin{proof}
Compactness of $\widetilde{\cA}$ and continuity implied by Lipschitzness ensure
that a maximizer $a^\star$ exists.  Since
$\cA_\epsilon\subseteq\widetilde{\cA}$,
\[
  \max_{\tilde a\in\cA_\epsilon}Q_r(s,\tilde a)
  \leq
  Q_r(s,a^\star),
\]
which proves the lower bound.  By the covering property, choose
$a_\epsilon\in\cA_\epsilon$ such that
$\|a^\star-a_\epsilon\|\leq\epsilon$.  Then
\begin{align*}
  Q_r(s,a^\star)
  -\max_{\tilde a\in\cA_\epsilon}Q_r(s,\tilde a)
  &\leq Q_r(s,a^\star)-Q_r(s,a_\epsilon)\\
  &\leq
  |Q_r(s,a^\star)-Q_r(s,a_\epsilon)|\\
  &\leq L_r\|a^\star-a_\epsilon\|
  \leq L_r\epsilon.
\end{align*}
The first inequality uses the fact that the prototype maximum is at least the
score of the particular prototype $a_\epsilon$.
\end{proof}

This discretization error is separate from quotient estimation error.  The
experiments above evaluate action ranking within the prototype set and do not
estimate $L_r\epsilon$ relative to the full continuous action space.

\begin{theorem}[Decision regret from measure error]
\label{thm:regret}
Let $\mathcal R$ be a class of bounded rewards and define
\[
  \|\widehat K-K\|_{\mathcal R^\star}
  =
  \sup_{r\in\mathcal R}
  \left|\int r\,\mathrm d(\widehat K-K)\right|.
\]
If
$\|\widehat K_{s,a}-K_{s,a}\|_{\mathcal R^\star}\leq\epsilon$
for every $a$, then for any $r\in\mathcal R$, the action
$\widehat a\in\argmax_a\int r\,\mathrm d\widehat K_{s,a}$ satisfies
\begin{equation}
  Q_r(s,a^\star)-Q_r(s,\widehat a)\leq2\epsilon,
  \label{eq:regret-bound}
\end{equation}
where $a^\star\in\argmax_a Q_r(s,a)$.
\end{theorem}

\begin{proof}
By \cref{thm:duality}, subtracting the reference value from $Q_r$ does not
change either action ordering or regret.  Define the true and estimated
centered scores
\[
  A(a)=\int r\,\mathrm dK_{s,a},
  \qquad
  \widehat A(a)=\int r\,\mathrm d\widehat K_{s,a}.
\]
The assumption implies
$|\widehat A(a)-A(a)|\leq\epsilon$ for every action.  Since
$\widehat a$ maximizes $\widehat A$, we have
$\widehat A(a^\star)-\widehat A(\widehat a)\leq0$.  Add and subtract these
two estimated scores:
\begin{align*}
  Q_r(s,a^\star)-Q_r(s,\widehat a)
  &=
  \int r\,\mathrm d(K-\widehat K)_{s,a^\star}
  +\bigl[
    \int r\,\mathrm d\widehat K_{s,a^\star}
    -\int r\,\mathrm d\widehat K_{s,\widehat a}
  \bigr]\\
  &\quad+
  \int r\,\mathrm d(\widehat K-K)_{s,\widehat a}.
\end{align*}
The first and third terms are each at most $\epsilon$ in absolute value, while
the bracketed term is nonpositive.  Therefore the whole expression is at most
$\epsilon+0+\epsilon=2\epsilon$.
\end{proof}

\begin{corollary}[Feature-level regret]
\label{cor:feature-regret}
If $\|w\|_2\leq W$ and
$\|\widehat\kappa(s,a)-\kappa(s,a)\|_2\leq\delta$ for every $a$, choosing
$\widehat a=\argmax_a w^\top\widehat\kappa(s,a)$ incurs regret at most
$2W\delta$ for reward $r_w=w^\top\phi$.
\end{corollary}

\begin{proof}
For each action, the feature-score error satisfies
\[
  |w^\top\widehat\kappa(s,a)-w^\top\kappa(s,a)|
  =
  |w^\top(\widehat\kappa-\kappa)(s,a)|
  \leq
  \|w\|_2\|\widehat\kappa-\kappa\|_2
  \leq W\delta,
\]
where the first inequality is Cauchy--Schwarz.  By
\cref{thm:duality}, these true feature scores are the centered action values
for $r_w$.  Repeating the optimal-action decomposition in
\cref{thm:regret}, the error $W\delta$ is paid once for $a^\star$ and once for
$\widehat a$, while estimated optimality contributes a nonpositive term.
The regret is therefore at most $2W\delta$.
\end{proof}

\begin{proposition}[Reward projection error]
\label{prop:reward-error}
Let $r$ be bounded and suppose
$\sup_{z\in\cZ}|r(z)-w^\top\phi(z)|\leq\epsilon_r$.  Then
\begin{equation}
  \left|
    \int r\,\mathrm dK^\pi_{H,s,a}
    -
    w^\top\kappa^\pi_H(s,a)
  \right|
  \leq 2m_H\epsilon_r.
  \label{eq:reward-error}
\end{equation}
In the infinite-horizon discounted case, the bound is
$2\epsilon_r/(1-\gamma)$.
\end{proposition}

\begin{proof}
Both $M^\pi_{H,s,a}$ and $\overline M^\pi_{H,s}$ are positive measures of
mass $m_H$.  The triangle inequality for total variation therefore gives
\[
  \|K^\pi_{H,s,a}\|_{\mathrm{TV}}
  =
  \|M^\pi_{H,s,a}-\overline M^\pi_{H,s}\|_{\mathrm{TV}}
  \leq
  \|M^\pi_{H,s,a}\|_{\mathrm{TV}}
  +\|\overline M^\pi_{H,s}\|_{\mathrm{TV}}
  =
  2m_H.
\]
By \cref{eq:feature-projection},
$w^\top\kappa^\pi_H(s,a)=\int w^\top\phi\,\mathrm dK^\pi_{H,s,a}$.
Subtracting this identity from the reward integral and applying the
total-variation inequality yields
\[
  \left|
    \int r\,\mathrm dK^\pi_{H,s,a}
    -w^\top\kappa^\pi_H(s,a)
  \right|
  =
  \left|\int(r-w^\top\phi)\,\mathrm dK^\pi_{H,s,a}\right|
  \leq
  \|r-w^\top\phi\|_\infty
  \|K^\pi_{H,s,a}\|_{\mathrm{TV}}
  \leq
  2m_H\epsilon_r.
\]
For an infinite discounted horizon, the mass of each positive successor
measure is the geometric sum
$\sum_{k=0}^{\infty}\gamma^k=(1-\gamma)^{-1}$.  Substituting this mass for
$m_H$ gives the stated infinite-horizon bound.
\end{proof}

\begin{corollary}[Combined estimation and reward-approximation regret]
\label{cor:combined-regret}
Assume $\|w\|_2\leq W$,
$\|\widehat\kappa(s,a)-\kappa(s,a)\|_2\leq\delta$ for every action, and
$\|r-w^\top\phi\|_\infty\leq\epsilon_r$.  If
$\widehat a\in\argmax_a w^\top\widehat\kappa(s,a)$, then
\begin{equation}
  Q_r(s,a^\star)-Q_r(s,\widehat a)
  \leq
  2W\delta+4m_H\epsilon_r.
  \label{eq:combined-regret}
\end{equation}
\end{corollary}

\begin{proof}
Define the true centered reward score
$A_r(a)=\int r\,\mathrm dK_{s,a}$ and the learned linear score
$\widehat A_w(a)=w^\top\widehat\kappa(s,a)$.  For every action, insert the
exact feature score $w^\top\kappa(s,a)$ and apply the triangle inequality:
\begin{align*}
  |A_r(a)-\widehat A_w(a)|
  &\leq
  \left|\int(r-w^\top\phi)\,\mathrm dK_{s,a}\right|
  +
  \left|w^\top(\kappa-\widehat\kappa)(s,a)\right|\\
  &\leq 2m_H\epsilon_r+W\delta,
\end{align*}
where \cref{prop:reward-error} bounds the first term and Cauchy--Schwarz bounds
the second.  Let
$\eta=2m_H\epsilon_r+W\delta$.  Because $\widehat a$ maximizes
$\widehat A_w$,
\[
  \widehat A_w(a^\star)-\widehat A_w(\widehat a)\leq0.
\]
Consequently,
\begin{align*}
  Q_r(s,a^\star)-Q_r(s,\widehat a)
  &=
  A_r(a^\star)-A_r(\widehat a)\\
  &=
  [A_r(a^\star)-\widehat A_w(a^\star)]
  +[\widehat A_w(a^\star)-\widehat A_w(\widehat a)]\\
  &\quad+
  [\widehat A_w(\widehat a)-A_r(\widehat a)]\\
  &\leq \eta+0+\eta
  =
  2W\delta+4m_H\epsilon_r.
\end{align*}
The first equality again uses the fact that the reference value is common to
all actions.
\end{proof}

\paragraph{Interpretation of the bound.}
Equation~\eqref{eq:combined-regret} separates two failure modes.  The term
$W\delta$ is an estimation error: even a reward exactly linear in $\phi$ is
misranked when the learned quotient is inaccurate.  The term
$2m_H\epsilon_r$ is a representation error: even a perfect feature quotient
cannot exactly answer a reward outside the span of $\phi$.  Increasing the
feature dimension may reduce $\epsilon_r$ but can increase statistical error
$\delta$ and weaken the low-rank advantage.  The measure-level model
corresponds conceptually to eliminating this finite-feature projection error,
but estimating a rich signed measure is itself more demanding.

\begin{proposition}[Dual recovery]
\label{prop:dual-recovery}
Let $w_1,\ldots,w_d$ form a basis of $\R^d$.  Exact access to the centered
values $A_{w_j,\rho}(s,a)=w_j^\top\kappa(s,a)$ for all $j$ uniquely determines
$\kappa(s,a)$.  Hence a sufficiently expressive task-conditioned value model
trained on a spanning reward family can implicitly recover the same decision
object as \CQM{}.
\end{proposition}

\begin{proof}
Form the square matrix
$W=[w_1,\ldots,w_d]^\top$, whose $j$th row is $w_j^\top$.  Since the vectors
$w_1,\ldots,w_d$ form a basis, the rows are linearly independent; therefore
$W$ is invertible.  Stack the observed centered values into
\[
  v(s,a)
  =
  [A_{w_1,\rho}(s,a),\ldots,A_{w_d,\rho}(s,a)]^\top.
\]
By the assumed readout identity, the $j$th coordinate of this vector is
$w_j^\top\kappa(s,a)$, so all coordinates together satisfy
$v(s,a)=W\kappa(s,a)$.  Left-multiplication by $W^{-1}$ gives the unique
solution
\[
  \kappa(s,a)=W^{-1}v(s,a).
\]
Thus no two distinct effect vectors can agree on every query from the spanning
reward family.
\end{proof}

\Cref{prop:dual-recovery} clarifies the intended comparison with universal
value function approximators: \CQM{} does not claim that reward-aware critics
cannot represent action effects.  It learns the primal effect once without
reward labels; a task-conditioned critic can learn the dual map when given a
sufficiently rich set of supervised reward queries.

\section{Limitations}

The empirical evidence remains controlled despite using physics-based
environments.  We use state observations, append a simulated
action-independent process, and evaluate random linear rewards rather than the
native DM Control objectives.  The results establish that quotient targeting
can improve short-horizon action ranking in these four mixed-state domains;
they do not establish native-task control performance, pixel-level robustness,
or gains in physical robotics, multi-agent systems, or real-world data.

The current training procedure assumes synchronized counterfactual branches:
the environment can be reset to the same state and rolled out with shared
random numbers and shared continuation actions.  This is natural in a
simulator but generally unavailable in a physical system or a fixed
observational dataset.  A learned generative model could provide approximate
branches, but model bias may then contaminate the quotient.  The
ordinary-trajectory density-ratio form in \cref{eq:hindsight-ratio} requires
absolute continuity, sufficient action coverage, and a reliable
outcome-conditioned classifier; it is not evaluated here.  Establishing its
behavior in large replay datasets remains open.

The measure-level quotient is sufficient for bounded measurable rewards, but
the implemented finite-dimensional projection is exact only for rewards
linear in the selected outcome features.  General rewards incur the
approximation error in \cref{prop:reward-error}.  Expanding the feature map
may reduce this error while increasing the action-effect dimension and
weakening the capacity advantage.  Similarly, the quotient is defined
relative to a continuation policy.  A model trained under one policy need not
remain accurate after a large policy shift; iterating model and policy
learning requires further analysis.

The current architecture enumerates a finite action set and centers over all
actions.  Continuous-action extensions require an integral under a reference
distribution or a learned action sampler, introducing approximation error.
The paired data cost also scales linearly with the number of discrete
actions.  In the present experiments, the three or five action prototypes are
only a coarse subset of each environment's continuous control space.

Only five seeds are reported.  Although the separation from absolute models
is large on Cartpole and Reacher, comparisons among \CQM{} and reward-aware
critics often have overlapping confidence intervals.  Walker shows that the
fixed rank-eight bottleneck can fail to recover the quotient as native state
dimension and dynamical complexity increase.  The 32 training reward
directions may also make held-out interpolation favorable for universal value
models.  Experiments varying reward supervision, effect rank, common-mode
dimension, branch coupling, horizon, and continuation policy are needed,
together with stronger probabilistic and recurrent world models.

Finally, \CQM{} intentionally omits action-independent predictive
information.  That information may be useful for objectives other than the
current decision query, including anomaly detection, representation
reuse for perception, safety monitoring, or planning under a changed action
set.  The quotient should therefore be viewed as a task-family-specific
decision model, not as a universal replacement for generative world models.

\section{Conclusion}

We proposed Counterfactual Quotient Models as a two-level representation of
action-conditioned futures.  The general object is a zero-total-mass signed
successor measure that retains what choices change and removes what all
choices share.  It is the coarsest exact object for all bounded-reward action
comparisons.  A finite feature projection yields the practical low-rank model
and exposes explicit guarantees for reward approximation, effect rank, and
decision regret.  A paired-rollout estimator cancels common random futures
before approximation, while a density-ratio identity gives a route to
ordinary trajectories.

Across four DM Control physics domains with high-dimensional common dynamics,
this change of learned object produces a large gap over equally sized
absolute world and successor-feature models on Cartpole and Reacher, a
smaller but consistent advantage on Cheetah, and only a narrow gain on Walker.
\CQM{} also competes with reward-aware critics despite receiving no reward
direction in its representation loss.  These results do not establish broad
control superiority, but they demonstrate the central point under a clearly
specified query family: an agent need not model everything that will happen
in order to model the action-dependent information needed for choice.
Scaling the quotient to ordinary replay data, richer feature maps, continuous
actions, and native closed-loop objectives is the next step.
\bibliography{iclr2027_conference}
\bibliographystyle{iclr2027_conference}

\end{document}